\documentclass[10pt,journal,compsoc]{IEEEtran}

\usepackage{amsmath,amsfonts}
\usepackage{amssymb}
\usepackage{amsthm}
\usepackage{algorithm}
\usepackage{algorithmic}
\newcommand{\Statex}{\item[]}
\newcommand{\AlgIO}[2]{%
  \Statex \hspace*{-0.85\labelwidth}\hspace*{-\labelsep}\textbf{#1:} #2
}

\usepackage{array}
\usepackage[caption=false,font=normalsize,labelfont=sf,textfont=sf]{subfig}
\usepackage{textcomp}
\usepackage{stfloats}
\usepackage{url}
\usepackage{verbatim}
\usepackage{graphicx}
\usepackage{booktabs}
\ifCLASSOPTIONcompsoc
  \usepackage[nocompress]{cite}
\else
  \usepackage{cite}
\fi

\usepackage{multirow}

\newtheorem{theorem}{Theorem}
\newtheorem{lemma}[theorem]{Lemma}

\begin{document}

\title{CARE-LoRA: Compressed Activation REconstruction for Memory-Efficient LoRA}

\author{Gengyu~Zhang, Haiyin~Ran, Zhengbao~He, Yuhang~Liu, Hanling~Tian, Zhehao~Huang,
        and~Xiaolin~Huang,~\IEEEmembership{Senior~Member,~IEEE}%
\thanks{The authors are with the Institute of Image Processing and Pattern Recognition, Shanghai Jiao Tong University, Shanghai 200240, China. E-mail: \{zhang2021jiaoda, rhy2006, lstefanie, yuhangliu, hanlingtian, kinght\_h, \mbox{xiaolinhuang}\}@sjtu.edu.cn.}%
\thanks{(Corresponding author: Xiaolin Huang.)}}

\IEEEtitleabstractindextext{%
\begin{abstract}
As the scale of large pre-trained models continues to grow, fine-tuning them under limited memory budgets has become increasingly challenging. Low-Rank Adaptation (LoRA), currently one of the most widely adopted parameter-efficient fine-tuning (PEFT) methods, mitigates this challenge by optimizing only low-rank adaptation matrices, thereby greatly reducing the number of trainable parameters. 
With the parameter overhead substantially reduced, the activations retained for backpropagation have emerged as the primary remaining memory bottleneck during LoRA fine-tuning.
To address this, we propose CARE-LoRA, a data-aware Compressed Activation REconstruction framework. By exploiting the inherent projection structure of LoRA, CARE-LoRA replaces the full input activation with the low-rank compressed activation naturally produced by the LoRA branch. It further computes a lightweight reconstruction matrix during the forward pass with negligible additional computation cost, which is used during backpropagation to reconstruct the gradient signal, thereby keeping LoRA matrices fully trainable.
Extensive experiments across diverse models and downstream tasks demonstrate that, while substantially reducing the overall memory footprint, CARE-LoRA achieves competitive or even superior performance compared with standard LoRA and representative LoRA variants. Our code is publicly available at \url{https://github.com/fishandyu/CARE-LoRA}.
\end{abstract}

}

\maketitle

\IEEEdisplaynontitleabstractindextext

\IEEEpeerreviewmaketitle

\IEEEraisesectionheading{\section{Introduction}\label{sec:introduction}}

\IEEEPARstart{P}{retraining} followed by downstream fine-tuning has become a standard recipe for building modern language systems. However, full-parameter fine-tuning remains difficult under limited GPU memory, because it not only updates all pretrained weights, but also requires storing their gradients, optimizer states, and intermediate activations for backpropagation. Parameter-efficient fine-tuning (PEFT) mitigates the parameter bottleneck by restricting optimization to a small number of trainable parameters. In PEFT, Low-Rank Adaptation (LoRA) has become a deployment-friendly and widely used approach: it freezes the pretrained backbone, learns only low-rank weight updates, and can merge the learned adaptation into the base model without extra inference latency \cite{hu2022lora}.

In practice, although LoRA substantially reduces trainable parameters, their gradients, and optimizer states, it does not reduce the activation memory required by backpropagation, making activations a more prominent bottleneck. This has motivated recent work on reducing the activations saved during LoRA-style fine-tuning. HyC-LoRA targets activations buffered for nonlinear operators in LoRA training and compresses them via outlier-aware low-bit quantization~\cite{wang2025hyclora}. LoRAct instead reduces activation memory by decomposing full activation matrices into compact low-rank factors~\cite{shi2025loract}. Gradient checkpointing, a standard memory-saving technique, reduces activation memory by storing only a subset of intermediate activations and recomputing the rest during backpropagation~\cite{chen2016training}. While effective, these methods mainly reduce activation memory through generic quantization, decomposition, or recomputation, rather than exploiting the structural properties of LoRA factors.

A representative approach that exploits this factorized structure is LoRA-FA~\cite{zhang2023lorafa}, which freezes $A$ and only updates $B$, thereby replacing stored full activations with compressed low-rank activations. Although LoRA-FA largely reduces the activation memory, its performance is far from satisfactory. The reason is that when $A$ is fixed, the weight change $\Delta W = AB$ is restricted to the subspace determined by a random, non-trained initialization matrix. In other words, the adapter obtained by LoRA-FA can only learn how to map features from a fixed low-dimensional subspace, but the subspace itself cannot be adjusted as in standard LoRA.

The challenge to overcome this limitation is that on the one hand we need to unfreeze $A$ to pursue adaptability and on the other hand we have to avoid calculating the exact gradient of $A$ to keep the memory efficiency. To achieve this goal, we propose a novel method to approximately reconstruct the backward information. We name this method \textbf{CARE-LoRA}, standing for the \textbf{C}ompressed \textbf{A}ctivation \textbf{RE}construction framework for memory-efficient low-rank adaptation. It stores the natural compressed activation $Z=XA\in\mathbb{R}^{N\times r}$ together with a lightweight \emph{data-aware} reconstruction matrix $M\in\mathbb{R}^{r\times m}$, where $r \ll \min\{m,N\}$. As a result, its memory overhead is marginally higher than LoRA-FA, while recovering much of the performance of standard LoRA by keeping $A$ trainable.

We evaluate CARE-LoRA across diverse fine-tuning settings, including natural language understanding, mathematical reasoning, code generation, instruction following, and image generation. Across these settings, our method reduces total peak memory by up to roughly $20\%$ relative to standard LoRA. At the same rank, it substantially improves over LoRA-FA, raising the average score by 4.3 points on GLUE~\cite{wang2018glue} and 7.4 points on SuperGLUE~\cite{wang2019superglue} at less than $1\%$ additional memory. Moreover, reducing memory consumption allows CARE-LoRA to use a higher rank for stronger performance. By increasing the rank, CARE-LoRA further surpasses LoRA and representative LoRA variants such as DoRA~\cite{Liu2024DoRAWL} and PiSSA~\cite{meng2024pissa}. Notably, CARE-LoRA achieves the highest average performance on the Mistral-7B-v0.3~\cite{jiang2023mistral} fine-tuning tasks while requiring only $1.02\times$ the per-step training time of LoRA and still maintaining a lower memory footprint than LoRA.

\begin{figure*}[!t]
    \centering
    \includegraphics[width=0.87\textwidth, trim={0 80pt 0 65pt}, clip]{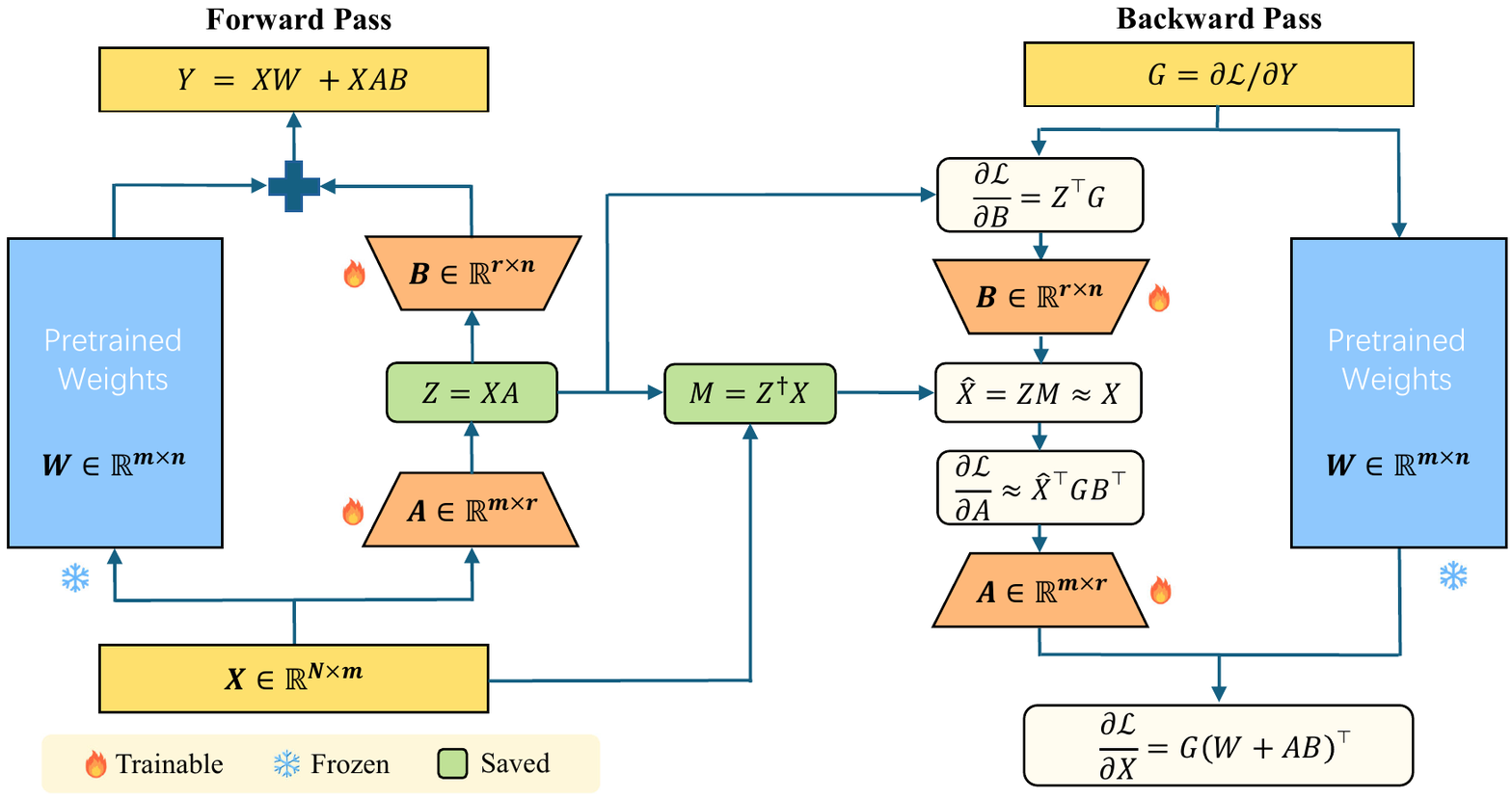}
    \caption{Overview of CARE-LoRA. CARE-LoRA stores the compressed activation $Z$ and a lightweight reconstruction matrix $M$ instead of the full activation $X$, enabling the gradient of $A$ to be reconstructed while keeping both LoRA factors trainable.}
    \label{fig:CARELORA}
\end{figure*}

The main contributions of this work are as follows:
\begin{itemize}
    \item We revisit the memory bottleneck of LoRA fine-tuning and note that, once adapter-side parameters are sufficiently reduced, activation storage becomes the dominant remaining obstacle.
    \item We propose CARE-LoRA, a data-aware framework that stores compressed activations and a lightweight reconstruction matrix, keeping both LoRA factors trainable under a near LoRA-FA memory budget. We further prove that its projection-down subspace can evolve during fine-tuning.
    \item We show empirically that CARE-LoRA substantially outperforms LoRA-FA at the same rank. By reinvesting the saved memory into a higher rank, it further surpasses standard LoRA and representative variants, achieving a favorable balance of memory efficiency, speed, and accuracy across multiple tasks.
\end{itemize}

\section{Related Work}

\subsection{Low-Rank Adaptation}

Low-Rank Adaptation (LoRA) is a widely adopted parameter-efficient fine-tuning method that freezes the pretrained weight $W$ and learns a low-rank update $\Delta W = AB$, where the projection-down matrix $A$ compresses the input activation $X$ and the projection-up matrix $B$ expands it back, adding no inference cost once merged \cite{hu2022lora}. Many variants improve LoRA from complementary perspectives. On the parameter side, AdaLoRA adaptively allocates the rank budget across weight matrices \cite{Zhang2023AdaLoRAAB}, while VeRA and LoRA-XS reduce trainable parameters further through shared frozen random projections \cite{Kopiczko2023VeRAVR} or a small trainable core between frozen SVD-derived factors \cite{Balazy2025LoRAXSLA}. On the optimization side, LoRA-GA and PiSSA enhance adapter initialization to speed up convergence \cite{wang2024loraga,meng2024pissa}, and DoRA decomposes each weight into magnitude and direction \cite{Liu2024DoRAWL}. QLoRA instead targets the frozen backbone, back-propagating through a 4-bit quantized model to cut base-weight memory \cite{dettmers2023qlora}. Despite these advances, all of them still retain the full input activation for backpropagation, leaving the activation memory of fine-tuning unaddressed.

\subsection{Memory-Efficient Fine-Tuning}

During fine-tuning, memory is consumed by the trainable parameters, their gradients, optimizer states, and the activations retained for backpropagation. Once parameter-efficient methods reduce the first three terms, activation becomes the dominant term and scales with batch size and sequence length \cite{shi2025loract}. General-purpose techniques target different parts of this budget, recomputing activations on the backward pass \cite{chen2016training,korthikanti2023reducing}, projecting gradients and optimizer states into a low-rank subspace \cite{zhao2024galore}, or compressing stored activations through quantization or low-rank projection \cite{Chen2021ActNNRT,shamshoum2025compact}, yet none of them is designed for LoRA fine-tuning. Closer to our setting, LoRAct compresses the cached activation with an online low-rank decomposition \cite{shi2025loract}. HyC-LoRA quantizes the activations buffered for nonlinear operators with an outlier-aware scheme \cite{wang2025hyclora}. Although both reduce the activation stored during LoRA fine-tuning, they do not make good use of the factorized structure of LoRA. LoRA-FA instead exploits this, freezing the projection-down matrix and training only the projection-up matrix, so it does not need to store the full activation \cite{zhang2023lorafa}. However, this restricts the projection-down subspace throughout fine-tuning and thus degrades performance. To address this limitation, we propose a method that uses the compressed activation naturally produced by LoRA to reconstruct the gradient for training. Our method not only preserves the memory savings of LoRA-FA but also retains the performance of standard LoRA, with only negligible computational overhead.

\section{Method}
\label{sec:method}

\begin{algorithm*}[t]
\caption{CARE-LoRA}
\label{alg:care_lora}
\fontsize{9pt}{11pt}\selectfont

\vspace{0.1em}
\begin{minipage}[t]{0.47\textwidth}
\raggedright
\textbf{\textit{Forward Pass}}\\[-0.7em]
\rule{\linewidth}{0.4pt}
\end{minipage}
\hfill
\begin{minipage}[t]{0.47\textwidth}
\raggedright
\textbf{\textit{Backward Pass}}\\[-0.7em]
\rule{\linewidth}{0.4pt}
\end{minipage}

\vspace{0.3em}

\begin{minipage}[t]{0.47\textwidth}
\raggedright
\begin{algorithmic}[1]
\AlgIO{Input} \text{activation} $X\in\mathbb{R}^{N\times m}$
\AlgIO{Storage} \text{frozen} weight $W\in\mathbb{R}^{m\times n}$, trainable matrices $A\in\mathbb{R}^{m\times r}$ and $B\in\mathbb{R}^{r\times n}$
\AlgIO{Output} $Y \in\mathbb{R}^{N\times n}$, matrices $Z,M$ to be saved

\STATE $Z \gets XA$ 
\Statex \hspace{\algorithmicindent}\textit{Prepare for backward reconstruction.}
\STATE $M \gets Z^{\dagger}X$
\Statex \hspace{\algorithmicindent}$Z^{\dagger}$\textit{ implemented as } $(Z^\top Z+\lambda I)^{-1}Z^\top$
\STATE $\Delta Y \gets ZB$
\STATE $Y \gets XW+\Delta Y$
\STATE Save $Z$ and $M$, discard $X$
\STATE Continue forward propagation with $Y$
\end{algorithmic}
\end{minipage}
\hfill
\vrule
\hfill
\begin{minipage}[t]{0.47\textwidth}
\raggedright
\begin{algorithmic}[1]
\AlgIO{Input} \text{upstream} gradient $G=\partial\mathcal{L}/\partial Y$
\AlgIO{Storage} \text{compressed} activation $Z \in\mathbb{R}^{N\times r}$, reconstruction matrix $M\in\mathbb{R}^{r\times m}$
\AlgIO{Output} \text{gradients} $\widetilde{\nabla}_A\mathcal{L}, \nabla_B\mathcal{L}, \nabla_X\mathcal{L}$

\STATE $\nabla_X\mathcal{L} \gets G(W+AB)^\top$
\STATE $\nabla_B\mathcal{L} \gets Z^\top G$
\Statex \hspace{\algorithmicindent}\textit{Exact gradient}
\STATE $\widehat{X} \gets ZM$
\Statex \hspace{\algorithmicindent}\textit{Activation reconstruction}
\STATE $\widetilde{\nabla}_A\mathcal{L} \gets \widehat{X}^{\top}GB^\top$
\STATE Update $A$ and $B$ with optimizer

\STATE Continue backpropagation with $\nabla_X\mathcal{L}$

\end{algorithmic}
\end{minipage}

\end{algorithm*}

\subsection{Preliminaries}
\label{subsec:lora_gradients}

Consider a linear layer with a frozen pretrained weight $W\in\mathbb{R}^{m\times n}$. Let $X\in\mathbb{R}^{N\times m}$ be the input activation, where $N$ denotes the number of tokens in the current batch, e.g., batch size times sequence length. LoRA introduces two trainable low-rank matrices $A\in\mathbb{R}^{m\times r}$ and $B\in\mathbb{R}^{r\times n}$, with $r\ll \min(m,n,N)$, and the adapted linear transformation is
\begin{equation}
    Y = X(W + AB).
    \label{eq:lora_forward}
\end{equation}
Following common practice, we omit the LoRA scaling factor
$s=\alpha/r$ for simplicity, since it can be absorbed into the low-rank
factors $A$ and $B$ without affecting the following derivation.

Denote the upstream gradient by
\begin{equation}
    G = \frac{\partial \mathcal{L}}{\partial Y}
    \in \mathbb{R}^{N\times n}.
\end{equation}
The exact LoRA gradients are
\begin{align}
    \nabla_B \mathcal{L}
    &= A^\top X^\top G
     = (XA)^\top G,
    \label{eq:grad_B_exact}
    \\
    \nabla_A \mathcal{L}
    &= X^\top G B^\top,
    \label{eq:grad_A_exact}
    \\
    \nabla_X \mathcal{L}
    &= G(W+AB)^\top.
    \label{eq:grad_X_exact}
\end{align}
Eq.~\eqref{eq:grad_B_exact} shows that the gradient of $B$ only requires an intermediate matrix
\begin{equation}
    Z = XA \in \mathbb{R}^{N\times r}.
    \label{eq:z_def}
\end{equation}
Since $Z$ is produced by the projection-down matrix $A$ and reduces the activation dimension from $m$ to $r$, it naturally serves as a compressed representation of the full activation $X$ under the LoRA parameterization.

In contrast, Eq.~\eqref{eq:grad_A_exact} shows that the gradient of $A$ still depends on the full activation $X\in\mathbb{R}^{N\times m}$. Therefore, standard LoRA must retain $X$ during the forward pass if both $A$ and $B$ are trainable. LoRA-FA avoids this activation dependency by freezing $A$ and updating only $B$. In this case, the backward pass only needs to store $Z=XA$ instead of the full activation $X$. However, the projection-down matrix $A$ then remains fixed throughout fine-tuning. CARE-LoRA keeps the same compressed activation $Z$, but reconstructs the missing activation signal needed to update $A$.

\subsection{Compressed Activation REconstruction}
\label{subsec:care_reconstruction}

CARE-LoRA aims to update both LoRA matrices while retaining the compressed activation $Z$ instead of the full activation $X$. The key challenge is computing the gradient of $A$. Since $\nabla_A\mathcal{L}=X^\top G B^\top$, directly reconstructing this gradient from $Z$ would require simultaneous access to the forward activation $X$ and the backward gradient $G$, which is incompatible with our activation-saving setting. We therefore reconstruct $X$ instead and use the reconstructed activation for subsequent gradient computation of $A$.

Since $Z=XA$ is already computed in the LoRA branch, we seek a linear reconstruction matrix
\begin{equation*}
    M \in \mathbb{R}^{r\times m}
\end{equation*}
such that the reconstructed activation
\begin{equation}
    \widehat{X} = ZM
    \label{eq:xhat_def}
\end{equation}
approximates the original activation $X$. For a fixed batch activation $X$ and current LoRA projection-down matrix $A$, we define $M^\star$ by the least-squares problem
\begin{equation}
\begin{aligned}
    M^\star
    &=
    \arg\min_{M}
    \|X - XAM\|_F^2 \\
    &=
    \arg\min_{M}
    \|X - ZM\|_F^2.
\end{aligned}
    \label{eq:least_squares_reconstruction}
\end{equation}
When $Z$ has full column rank, the closed-form solution is
\begin{equation}
    M^\star = (Z^\top Z)^{-1}Z^\top X = Z^\dagger X .
    \label{eq:m_star}
\end{equation}
Thus, under the Frobenius reconstruction loss, CARE-LoRA uses the current batch itself to compute the best linear decoder from the compressed activation $Z$ back to $X$.

In practice, directly inverting $Z^\top Z$ is numerically unstable, since $Z$ could be rank-deficient or ill-conditioned. We therefore use a Tikhonov-regularized version \cite{tikhonov1963regularization}
\begin{equation}
    M^\star_\lambda
    =
    (Z^\top Z+\lambda I_r)^{-1}Z^\top X ,
    \qquad \lambda>0.
    \label{eq:stable_solution}
\end{equation}
The regularizer makes the reconstruction numerically stable and always well-defined, because $Z^\top Z+\lambda I_r$ is positive definite for any $\lambda>0$. In our implementation, $\lambda$ is set to a very small value, serving only as a numerical safeguard with negligible effect on the least-squares reconstruction. For simplicity, we write $M$ for $M^\star_\lambda$ below.

During the forward pass, CARE-LoRA follows standard LoRA, and the difference lies in what is kept for the backward pass. While $X$ is still available during the forward computation, CARE-LoRA uses it once to compute the batch-dependent reconstruction matrix $M$ as in Eq.~\eqref{eq:stable_solution}, then saves only
\begin{equation*}
    Z\in\mathbb{R}^{N\times r}
    \quad\text{and}\quad
    M\in\mathbb{R}^{r\times m}
\end{equation*}
for the LoRA backward pass, and discards the full activation $X\in\mathbb{R}^{N\times m}$. Therefore, the saved LoRA activation-related tensors scale as $O(Nr+mr)$ instead of $O(Nm)$. Since $r\ll \min(N,m)$, this is much smaller than storing $X$.

During the backward pass, given the upstream gradient $G=\partial\mathcal{L}/\partial Y$, CARE-LoRA computes the gradient of $B$ exactly using the stored matrix $Z$, 
\begin{equation}
    \nabla_B\mathcal{L}=(XA)^\top G=Z^\top G,
    \label{eq:care_grad_b}
\end{equation}
which is identical to Eq.~\eqref{eq:grad_B_exact}. The input gradient $\nabla_X \mathcal{L}$ is also unchanged from standard LoRA and is computed as in Eq.~\eqref{eq:grad_X_exact}. Therefore, CARE-LoRA does not introduce any reconstruction error into the gradient propagated to previous layers.

The only approximated term is the gradient of $A$. Since Eq.~\eqref{eq:grad_A_exact} requires the discarded activation $X$, CARE-LoRA uses the reconstructed activation $\widehat{X}=ZM$ in its place: 
\begin{equation}
    \widetilde{\nabla}_A\mathcal{L}
    =
    \widehat{X}^{\top}GB^\top
    =
    M^\top\bigl(Z^\top(GB^\top)\bigr).
    \label{eq:care_grad_a}
\end{equation}
In our implementation, we compute the last expression in Eq.~\eqref{eq:care_grad_a} directly, without explicitly constructing the large matrix $\widehat{X}=ZM\in\mathbb{R}^{N\times m}$. This avoids forming a full-size reconstructed activation and further reduces peak memory.

Thus, CARE-LoRA reduces activation memory by storing $Z$ and $M$ instead of the full activation $X$, while the only approximation introduced during training is the local gradient used to update $A$. The update of $B$ and the inter-layer gradient propagation remain exact as in standard LoRA. Fig.~\ref{fig:CARELORA} illustrates the forward and backward computation, and Algorithm~\ref{alg:care_lora} formalizes it in detail.

\subsection{Adaptation of the Projection-Down Subspace}
\label{subsec:subspace_adaptation}

The main limitation of LoRA-FA is not only that one of the two LoRA matrices is frozen, but that the update subspace induced by the projection-down matrix $A$ remains fixed throughout fine-tuning. In a LoRA layer with $\Delta W=AB$, every column of the low-rank update satisfies
\begin{equation}
    \Delta W_{:,j} = A B_{:,j},
\end{equation}
and therefore lies in $\operatorname{Col}(A)$. Consequently, once LoRA-FA freezes $A$, the initialization of $A$ determines a fixed $r$-dimensional subspace in which all subsequent updates to $W$ must reside, while training $B$ can only adjust the coordinates of $\Delta W$ within this subspace.

Since $A\in\mathbb{R}^{m\times r}$ with $r\ll m$, generically $A$ has full column rank $r$, so its row space is already the entire $\mathbb{R}^r$. Therefore, the expressive bottleneck lies in the column space $\operatorname{Col}(A)$, an $r$-dimensional subspace of the much larger input feature space $\mathbb{R}^m$, determining which input directions are retained by the compressed activation $Z=XA$. To understand whether CARE-LoRA truly keeps the projection-down matrix adaptive, we analyze the evolution of $\operatorname{Col}(A)$ below.

For the following analysis, we let
\begin{equation*}
    U = G B^\top \in \mathbb{R}^{N\times r},
    \qquad
    \Sigma_X = X^\top X \in \mathbb{R}^{m\times m}.
\end{equation*}
Here $U$ is the upstream gradient after being propagated through $B^\top$, and $\Sigma_X$ is the unnormalized second-moment matrix of $X$, or equivalently the unnormalized batch covariance when the activations are centered.

Using the regularized reconstruction matrix in Eq.~\eqref{eq:stable_solution}, $\widetilde{\nabla}_A\mathcal{L}$ in Eq.~\eqref{eq:care_grad_a} can be written as
\begin{align}
    \widetilde{\nabla}_A\mathcal{L}
    &= M^\top Z^\top G B^\top \notag \\
    &= X^\top Z (Z^\top Z+\lambda I_r)^{-1} Z^\top U \notag \\
    &= X^\top X A (A^\top X^\top X A+\lambda I_r)^{-1} A^\top X^\top U \notag \\
    &= \Sigma_X A K_\lambda,
    \label{eq:care_grad_a_covariance_form}
\end{align}
where
\begin{equation}
    K_\lambda
    =
    (A^\top \Sigma_X A+\lambda I_r)^{-1} A^\top X^\top U
    \in \mathbb{R}^{r\times r}.
    \label{eq:k_lambda_def}
\end{equation}

Eq.~\eqref{eq:care_grad_a_covariance_form} exposes the source of subspace adaptation. The factor $\Sigma_X A$ first maps the current projection directions through the second-moment matrix, while $K_\lambda$ only forms linear combinations of the resulting $r$ directions. 
Therefore, CARE-LoRA can introduce directions outside the current column space of $A$ when $\Sigma_X$ maps the current subspace to directions not already contained in it, as long as the backward signal does not cancel these off-subspace components.

To clearly state this intuition, we let $S=\operatorname{Col}(A)$ be the current projection-down subspace, and let 
\begin{equation*}
    P_S \in \mathbb{R}^{m\times m}
    \qquad
    \text{and} 
    \qquad
    P_S^\perp = I_m - P_S 
\end{equation*}
denote the orthogonal projectors onto $S$ and its complement, respectively. Then consider a gradient step on $A$ using the reconstructed gradient with learning rate $\eta>0$, 
\begin{equation}
    A^+ = A - \eta \widetilde{\nabla}_A\mathcal{L}
        = A - \eta \Sigma_X A K_\lambda ,
    \label{eq:a_update_care}
\end{equation}
where $A^+$ denotes the updated projection-down matrix.
Since $P_S^\perp A=0$, we have
\begin{equation}
    P_S^\perp A^+
    =
    -\eta P_S^\perp \Sigma_X A K_\lambda.
    \label{eq:outside_component}
\end{equation}
Hence, if
\begin{equation}
    P_S^\perp \Sigma_X A K_\lambda \neq 0,
    \label{eq:subspace_change_condition}
\end{equation}
then $A^+$ has at least one column with a nonzero component outside $S$.
Therefore,
\begin{equation}
    \operatorname{Col}(A^+) \nsubseteq \operatorname{Col}(A),
\end{equation}
which implies that the column space of $A$ has changed. If $A^+$ remains full column rank, then this change is a movement of the learned $r$-dimensional projection-down subspace.

The condition in Eq.~\eqref{eq:subspace_change_condition} separates into a data-dependent part and a gradient-dependent part. The data-dependent part is
\begin{equation}
    P_S^\perp \Sigma_X A \neq 0.
    \label{eq:covariance_not_invariant}
\end{equation}
This condition fails if and only if
\begin{equation}
    \operatorname{Col}(\Sigma_X A)\subseteq \operatorname{Col}(A),
\end{equation}
or equivalently, when $\operatorname{Col}(A)$ is an invariant subspace of $\Sigma_X$. However, since $\Sigma_X=X^\top X$ is induced by the heterogeneous input data, a randomly initialized $r$-dimensional subspace $\operatorname{Col}(A)$ will not in general be invariant under $\Sigma_X$, except for some special cases such as $\Sigma_X=cI_m$.

The gradient-dependent part concerns whether the right factor $K_\lambda$ cancels the off-subspace component produced by $\Sigma_XA$. Let
\begin{equation}
    D = P_S^\perp \Sigma_X A .
\end{equation}
When the data-dependent condition holds, we have $D\neq 0$. The outside component disappears only in the degenerate case
\begin{equation}
    D K_\lambda = 0.
\end{equation}
This requires $K_\lambda$ to be rank-deficient. Otherwise, if $K_\lambda$ were nonsingular, right multiplication by $K_\lambda$ could not turn a nonzero matrix $D$ into zero. Since
\begin{equation}
    K_\lambda
    =
    (A^\top \Sigma_X A+\lambda I_r)^{-1} Z^\top U,
\end{equation}
where $A^\top \Sigma_X A+\lambda I_r$ is invertible for $\lambda>0$, this can happen only when $Z^\top U$ is rank-deficient.

The rank deficiency of $Z^\top U$ can arise in only three cases. First, $Z$ is rank-deficient, i.e., $\operatorname{rank}(Z)<r$. Second, $U$ is rank-deficient, i.e., $\operatorname{rank}(U)<r$. 
These two cases are already degenerate: both $Z$ and $U$ are matrices in $\mathbb{R}^{N\times r}$, and in typical fine-tuning batches we have $N\gg r$, so they are expected to have full column rank under non-degenerate activations and backward signals.

The remaining case is more subtle: both $Z$ and $U$ have rank $r$, but $Z^\top U$ is still singular. 
This is a nongeneric event. In particular, if $Z$ and $U$ are viewed as generic full-column-rank matrices in $\mathbb{R}^{N\times r}$, then
\begin{equation}
    \Pr\!\left(\operatorname{rank}(Z^\top U)<r
    \,\middle|\,
    \operatorname{rank}(Z)=\operatorname{rank}(U)=r
    \right)=0.
\end{equation}
Combining the three cases above, under the generic non-degenerate setting, $Z^\top U$ is not rank-deficient. Therefore, when the data-dependent condition $D\neq 0$ holds, we generically have $D K_\lambda \neq 0$, so
the gradient-dependent part retains an off-subspace component in the update of $A$.

In summary, CARE-LoRA avoids the fixed-subspace limitation of LoRA-FA.
The reconstructed gradient $\widetilde{\nabla}_A\mathcal{L}=\Sigma_X A K_\lambda$ uses the input second-moment matrix to generate data-dependent update directions for $A$. Under non-degenerate activations and backward signals, these directions are not restricted to $\operatorname{Col}(A)$, allowing the projection-down subspace to evolve while keeping the memory benefit of avoiding storage of the full activation $X$.

\subsection{Memory and Computation Costs}
\label{subsec:memory_computation_cost}

We next analyze the memory and computation cost introduced by CARE-LoRA.
The analysis focuses on the tensors and matrix multiplications specific to the LoRA branch of one adapted linear layer. Computation and storage shared by standard LoRA and CARE-LoRA, including the frozen-weight path of the adapted linear layer and intermediate tensors required by non-LoRA operations, are excluded from this local comparison.

\textbf{Memory Cost.}
For a LoRA layer with input activation $X\in\mathbb{R}^{N\times m}$, projection-down matrix $A\in\mathbb{R}^{m\times r}$, and projection-up matrix $B\in\mathbb{R}^{r\times n}$, standard LoRA must retain the full activation $X$ in order to compute $\nabla_A\mathcal{L}=X^\top G B^\top$. Thus, the LoRA-side saved activation tensor scales as
\begin{equation}
    \mathcal{M}_{\mathrm{LoRA}}^{\mathrm{act}}
    =
    Nm .
    \label{eq:lora_activation_memory}
\end{equation}
CARE-LoRA keeps both $A$ and $B$ trainable. It saves the compressed activation $Z\in\mathbb{R}^{N\times r}$ and the reconstruction matrix $M\in\mathbb{R}^{r\times m}$, but discards the full activation $X$. Therefore,
\begin{equation}
    \mathcal{M}_{\mathrm{CARE}}^{\mathrm{act}}
    =
    Nr + rm
    =
    r(N+m).
    \label{eq:care_activation_memory}
\end{equation}
The activation-memory ratio to standard LoRA is
\begin{equation}
    \frac{
    \mathcal{M}_{\mathrm{CARE}}^{\mathrm{act}}
    }{
    \mathcal{M}_{\mathrm{LoRA}}^{\mathrm{act}}
    }
    =
    \frac{r(N+m)}{Nm}.
    \label{eq:care_memory_ratio}
\end{equation}
Therefore, CARE-LoRA substantially reduces the LoRA-side activation storage when the rank $r\ll\min(m,N)$.

As a concrete example, consider fine-tuning T5-Base on MNLI with sequence length $256$ and batch size $32$, so $N=8192$. For a typical hidden-size linear layer with $m=768$ and rank $r=8$, the LoRA-side activation-memory ratio of CARE-LoRA to standard LoRA is
\begin{equation*}
    \frac{8(8192+768)}{8192\times768}
    \approx
    1.14\%.
\end{equation*}

\textbf{Computation Cost.}
We count the LoRA-branch matrix-multiplication cost for one adapted linear layer up to constant factors, including both the forward and backward passes. In standard LoRA, the LoRA branch involves three multiplications with cost $Nmr$ and three multiplications with cost $Nrn$, giving
\begin{equation}
    T_{\mathrm{LoRA}}
    =
    3Nmr + 3Nrn
    =
    3Nr(m+n).
    \label{eq:lora_compute_cost}
\end{equation}
The additional cost of CARE-LoRA mainly comes from solving the reconstruction matrix $M$ and using it to compute the reconstructed $A$-gradient, giving
\begin{equation}
    T_{\mathrm{extra}}
    =
    2Nr^2 + 2mr^2 + r^3 .
    \label{eq:care_extra_compute_cost}
\end{equation}
Thus,
\begin{equation}
    T_{\mathrm{CARE}}
    =
    T_{\mathrm{LoRA}} + T_{\mathrm{extra}} .
\end{equation}
The computation cost ratio to standard LoRA is
\begin{equation}
    \frac{T_{\mathrm{CARE}}}{T_{\mathrm{LoRA}}}
    =
    1+
    \frac{r(2N+2m+r)}{3N(m+n)} .
    \label{eq:care_compute_ratio}
\end{equation}
Since $r\ll\min(m,n,N)$, the additional reconstruction terms do not change the dominant order of LoRA-branch matrix multiplications. Therefore, CARE-LoRA introduces only a small computation overhead compared with standard LoRA.

For the same T5-Base setting, taking $m=n=768$ and $r=8$ as an example, the additional LoRA-branch computation cost of CARE-LoRA over standard LoRA is
\begin{equation*}
    \frac{8(2\times8192+2\times768+8)}{3\times8192(768+768)}
    \approx
    0.38\%.
\end{equation*}

\begin{table*}[!t]
\centering
\caption{Results (\%) and Memory Usage (MiB) on Fine-Tuning the T5-Base Model on a Subset of GLUE Datasets. 
\textit{Avg. Mem} Denotes the Average Peak CUDA Memory Usage Across Different Tasks.}
\label{tab:glue_results}
\resizebox{\textwidth}{!}{
\renewcommand{\arraystretch}{1.05}
\begin{tabular}{l|ccccc|cc}
\toprule
\textbf{Method} 
& \textbf{MNLI} 
& \textbf{QNLI} 
& \textbf{SST-2} 
& \textbf{CoLA} 
& \textbf{MRPC} 
& \textbf{Avg.} 
& \textbf{Avg. Mem} \\
\midrule
LoRA ($r=8$)
& $\mathbf{85.89}_{\pm 0.02}$ 
& $93.10_{\pm 0.08}$ 
& $\underline{94.80}_{\pm 0.13}$ 
& $\underline{60.39}_{\pm 2.07}$ 
& $92.04_{\pm 0.33}$ 
& $\mathbf{85.24}_{\pm 0.48}$ 
& $6310.66$ \\
LoRAct ($r=8$)
& $\underline{85.85}_{\pm 0.14}$ 
& $\underline{93.16}_{\pm 0.05}$ 
& $\mathbf{94.92}_{\pm 0.52}$ 
& $59.55_{\pm 0.81}$ 
& $\underline{92.22}_{\pm 0.32}$ 
& $85.14_{\pm 0.06}$ 
& $5185.06$ \\
LoRA-FA ($r=8$)
& $84.61_{\pm 0.16}$ 
& $92.76_{\pm 0.12}$ 
& $94.34_{\pm 0.18}$ 
& $46.08_{\pm 1.28}$ 
& $87.05_{\pm 0.95}$ 
& $80.97_{\pm 0.40}$ 
& $\mathbf{4988.06}$ \\
CARE-LoRA ($r=8$)
& $85.59_{\pm 0.11}$ 
& $\mathbf{93.19}_{\pm 0.13}$ 
& $94.46_{\pm 0.40}$ 
& $\mathbf{60.46}_{\pm 1.66}$ 
& $\mathbf{92.41}_{\pm 0.45}$ 
& $\underline{85.22}_{\pm 0.29}$ 
& $\underline{5019.37}$ \\
\bottomrule
\end{tabular}
}
\end{table*}

\begin{table*}[!t]
\centering
\caption{Results (\%) and Memory Usage (MiB) on Fine-Tuning the T5-Base Model on a Subset of SuperGLUE Datasets. 
\textit{Avg. Mem} Denotes the Average Peak CUDA Memory Usage Across Different Tasks.}
\label{tab:superglue_results}
\resizebox{\textwidth}{!}{
\renewcommand{\arraystretch}{1.05}
\begin{tabular}{l|ccccc|cc}
\toprule
\textbf{Method} 
& \textbf{BoolQ} 
& \textbf{CB} 
& \textbf{COPA} 
& \textbf{RTE} 
& \textbf{WiC} 
& \textbf{Avg.} 
& \textbf{Avg. Mem} \\
\midrule
LoRA ($r=8$)
& $\underline{81.01}_{\pm 0.51}$ 
& $\underline{88.03}_{\pm 3.34}$ 
& $\underline{70.00}_{\pm 3.00}$ 
& $\mathbf{81.47}_{\pm 1.37}$ 
& $68.55_{\pm 0.09}$ 
& $77.81_{\pm 0.32}$ 
& $7274.58$ \\
LoRAct ($r=8$)
& $\mathbf{81.48}_{\pm 0.53}$ 
& $87.44_{\pm 1.79}$ 
& $\mathbf{70.67}_{\pm 1.53}$ 
& $80.63_{\pm 0.75}$ 
& $\mathbf{69.17}_{\pm 0.86}$ 
& $\underline{77.88}_{\pm 0.60}$ 
& $5964.20$ \\
LoRA-FA ($r=8$)
& $79.14_{\pm 0.18}$ 
& $76.44_{\pm 4.92}$ 
& $58.00_{\pm 2.00}$ 
& $72.32_{\pm 0.55}$ 
& $68.13_{\pm 0.59}$ 
& $70.81_{\pm 1.00}$ 
& $\mathbf{5759.73}$ \\
CARE-LoRA ($r=8$)
& $81.00_{\pm 0.68}$ 
& $\mathbf{90.03}_{\pm 1.40}$ 
& $\underline{70.00}_{\pm 1.00}$ 
& $\underline{80.99}_{\pm 0.55}$ 
& $\underline{68.86}_{\pm 0.36}$ 
& $\mathbf{78.18}_{\pm 0.50}$ 
& $\underline{5791.07}$ \\
\bottomrule
\end{tabular}
}
\end{table*}


\section{Experiments}
\label{sec:experiments}

We evaluate CARE-LoRA on three categories of fine-tuning tasks. For \textit{natural language understanding} (NLU), we fine-tune T5-Base~\cite{raffel2020t5} on GLUE~\cite{wang2018glue} and SuperGLUE~\cite{wang2019superglue}. For \textit{natural language generation} (NLG), we fine-tune Mistral-7B-v0.3~\cite{jiang2023mistral} for mathematical reasoning, code generation, and instruction following. For \textit{diffusion-based image generation}, we fine-tune SD3-Medium~\cite{esser2024sd3} on DreamBooth~\cite{ruiz2023dreambooth}. Together, these tasks cover diverse domains, model architectures, and sequence lengths ranging from $256$ to $2048$ tokens, providing a comprehensive assessment of CARE-LoRA across a broad range of fine-tuning scenarios.

\subsection{Baselines}
\label{subsec:baselines}

We mainly compare the following methods:
\begin{itemize}
    \item \textbf{LoRA}~\cite{hu2022lora} applies low-rank adaptation to all linear layers of the frozen backbone.
    \item \textbf{LoRAct}~\cite{shi2025loract} compresses cached activations via low-rank decomposition. We apply it only to LoRA layers to match CARE-LoRA under a comparable memory budget.
    \item \textbf{LoRA-FA}~\cite{zhang2023lorafa} freezes $A$ and updates only $B$ to avoid storing full activations.
    \item \textbf{CARE-LoRA} (ours) keeps both $A$ and $B$ trainable while avoiding storing full activations.
\end{itemize}

For a fair comparison, all methods use the same backbone, training data, target modules, optimizer (AdamW~\cite{loshchilov2019adamw}), learning rate, precision, batch size, and training steps within each task, implemented in a unified PEFT-based codebase~\cite{mangrulkar2022peft}. In the LLM fine-tuning experiments, we additionally report CARE-LoRA at $r=16$, which is still more memory-efficient than LoRA and LoRAct at $r=8$. The only method-specific hyperparameter is decomposition rank $k$ of LoRAct, set to $64$ for T5-Base and $128$ for Mistral-7B-v0.3. All results are averaged over three common seeds ($0$, $21$, $42$) with standard deviation reported. The best result in each column is boldfaced and the second-best is underlined. Additional analyses are presented later in this section.


\begin{table*}[!t]
\centering
\caption{Results (\%) and Memory Usage (MiB) on Fine-Tuning the Mistral-7B-v0.3 Model on Three LLM Fine-Tuning Tasks. 
\textit{Avg. Mem} Denotes the Average Peak CUDA Memory Usage Across Different Tasks.}
\label{tab:llm_results}
\resizebox{\textwidth}{!}{
\renewcommand{\arraystretch}{1.05}
\begin{tabular}{l|cccc|cc}
\toprule
\multirow{2}{*}[-0.6ex]{\textbf{Method}}
& \multicolumn{1}{c}{\textbf{Math}} 
& \multicolumn{1}{c}{\textbf{Code}} 
& \multicolumn{2}{c|}{\textbf{Instruct}} 
& \multicolumn{2}{c}{\textbf{Overall}} \\
\cmidrule(lr){2-2}
\cmidrule(lr){3-3}
\cmidrule(lr){4-5}
\cmidrule(lr){6-7}
& \textbf{GSM8K} 
& \textbf{HumanEval} 
& \textbf{Prompt-Strict} 
& \textbf{Instruction-Strict} 
& \textbf{Avg.} 
& \textbf{Avg. Mem} \\
\midrule
LoRA ($r=8$)
& $71.22_{\pm 0.70}$ 
& $42.89_{\pm 2.54}$ 
& $45.22_{\pm 1.23}$ 
& $55.64_{\pm 0.42}$ 
& $53.74_{\pm 0.50}$ 
& $29214.27$ \\
PiSSA ($r=8$)
& $\underline{71.44}_{\pm 0.19}$ 
& $41.46_{\pm 2.44}$ 
& $45.04_{\pm 0.47}$ 
& $55.92_{\pm 0.70}$ 
& $53.47_{\pm 0.76}$ 
& $29208.54$ \\
DoRA ($r=8$)
& $70.99_{\pm 0.46}$ 
& $\mathbf{43.90}_{\pm 2.79}$ 
& $\underline{46.03}_{\pm 0.92}$ 
& $\underline{56.28}_{\pm 0.69}$ 
& $\underline{54.30}_{\pm 1.13}$ 
& $30701.65$ \\
\midrule
LoRAct ($r=8$)
& $70.91_{\pm 0.43}$ 
& $\underline{43.50}_{\pm 0.70}$ 
& $44.24_{\pm 1.05}$ 
& $55.00_{\pm 0.57}$ 
& $53.41_{\pm 0.51}$ 
& $25185.51$ \\
LoRA-FA ($r=8$)
& $67.65_{\pm 0.57}$ 
& $40.85_{\pm 0.61}$ 
& $39.49_{\pm 2.03}$ 
& $51.48_{\pm 2.00}$ 
& $49.87_{\pm 0.76}$ 
& $\mathbf{24214.74}$ \\
CARE-LoRA ($r=8$)
& $70.68_{\pm 0.12}$ 
& $42.68_{\pm 2.20}$ 
& $44.49_{\pm 0.28}$ 
& $54.64_{\pm 0.14}$ 
& $53.12_{\pm 0.46}$ 
& $\underline{24469.27}$ \\
CARE-LoRA ($r=16$)
& $\mathbf{71.82}_{\pm 0.23}$ 
& $43.29_{\pm 1.61}$ 
& $\mathbf{46.27}_{\pm 0.77}$ 
& $\mathbf{56.47}_{\pm 0.52}$ 
& $\mathbf{54.47}_{\pm 0.20}$ 
& $24948.65$ \\
\bottomrule
\end{tabular}
}
\end{table*}



\begin{table*}[!t]
\centering
\caption{Results (\%) and Memory Usage (MiB) on Fine-Tuning Diffusion Model SD3-Medium on DreamBooth, Averaged Over Ten Subjects and Three Seeds. 
\textit{Avg. Mem} Denotes the Average Peak CUDA Memory Usage.}
\label{tab:diffusion_results}
\small
\setlength{\tabcolsep}{12pt}
\renewcommand{\arraystretch}{1.05}
\begin{tabular}{l|ccc|cc}
\toprule
\textbf{Method} 
& \textbf{DINO} 
& \textbf{CLIP-I} 
& \textbf{CLIP-T} 
& \textbf{Avg.} 
& \textbf{Avg. Mem} \\
\midrule
LoRA ($r=8$)
& $43.18_{\pm 0.11}$ 
& $69.18_{\pm 0.25}$ 
& $\mathbf{28.01}_{\pm 0.05}$ 
& $46.79_{\pm 0.11}$ 
& $9882.30$ \\
CARE-LoRA ($r=16$)
& $\mathbf{43.40}_{\pm 0.15}$ 
& $\mathbf{69.29}_{\pm 0.37}$ 
& $28.00_{\pm 0.11}$ 
& $\mathbf{46.89}_{\pm 0.16}$ 
& $\mathbf{8630.47}$ \\
\bottomrule
\end{tabular}
\end{table*}


\subsection{Results on Natural Language Understanding}
\label{subsec:results_nlu}

\textbf{Setting.} We fine-tune T5-Base on five widely used tasks from each of GLUE (MNLI, QNLI, SST-2, CoLA, MRPC) and SuperGLUE (BoolQ, CB, COPA, RTE, WiC), following the original text-to-text formulation. We report each task's standard metric, accuracy for most tasks, Matthews correlation for CoLA, and F1 for MRPC and CB. Following the protocol described in Sec.~\ref{subsec:baselines}, all methods use rank $r=8$. Results are summarized in Tables~\ref{tab:glue_results} and~\ref{tab:superglue_results}.

\textbf{Results.} At the same rank, CARE-LoRA achieves strong performance while keeping a memory footprint close to LoRA-FA. On GLUE, CARE-LoRA obtains an average score of 85.22, essentially matching standard LoRA at 85.24 and outperforming LoRA-FA by 4.25 points. On SuperGLUE, CARE-LoRA achieves the best average score of 78.18, improving over LoRA and LoRAct while surpassing LoRA-FA by 7.37 points. In terms of memory, CARE-LoRA reduces average peak memory by about $20\%$ on both GLUE and SuperGLUE relative to LoRA, while requiring less than $1\%$ additional memory compared with LoRA-FA. These results show that the reconstructed update of $A$ effectively mitigates the fixed-subspace limitation of LoRA-FA without giving up its activation-memory advantage.

\subsection{Results on Large Language Models}
\label{subsec:results_llm}

\textbf{Setting.} We evaluate CARE-LoRA on Mistral-7B-v0.3 across three tasks: mathematical reasoning, code generation, and instruction following. Each task focuses on a specific capability and uses well-established datasets and metrics for training and evaluation, as detailed below:
\begin{itemize}
    \item \textbf{Mathematical reasoning.} We fine-tune on a 100k subset of MetaMathQA~\cite{yu2024metamath} with sequence length under $512$ tokens and report accuracy on GSM8K~\cite{cobbe2021gsm8k}.
    \item \textbf{Code generation.} We fine-tune on a 100k subset of OpenCodeInstruct~\cite{ahmad2025opencodeinstruct} with sequence length under $1024$ tokens and report pass@1 on HumanEval~\cite{chen2021codex}.
    \item \textbf{Instruction following.} We fine-tune on a 50k subset of SmolTalk~\cite{allal2025smollm2} with sequence length under $2048$ tokens and report prompt-level and instruction-level strict accuracy on IFEval~\cite{zhou2023ifeval}.
\end{itemize}

Here, we additionally compare against two widely used LoRA variants that target performance improvement rather than memory efficiency. DoRA~\cite{Liu2024DoRAWL} decomposes pretrained weights into magnitude and direction, and we use its memory-saving version in our experiments. PiSSA~\cite{meng2024pissa} initializes the LoRA factors from the principal components of the pretrained weight via SVD. Since CARE-LoRA reduces activation memory, the saved memory can be reinvested into a higher rank. We therefore additionally evaluate CARE-LoRA at $r=16$ as a performance-oriented setting for comparison with other baselines.

\textbf{Results.} Table~\ref{tab:llm_results} reports results on Mistral-7B-v0.3. At $r=8$, CARE-LoRA achieves competitive performance compared with standard LoRA while substantially reducing peak memory. At $r=16$, CARE-LoRA achieves the highest overall average of 54.47 and leads all methods on GSM8K, Prompt-Strict, and Instruction-Strict, while also outperforming LoRA on HumanEval. DoRA achieves the second-best overall average, yet CARE-LoRA requires 18.7\% less peak memory. LoRA-FA attains the lowest memory footprint by freezing $A$ throughout fine-tuning, but CARE-LoRA uses only 3\% more memory while improving the overall average by 4.6 points. Excluding LoRA-FA, CARE-LoRA at $r=16$ achieves the best accuracy while using less peak memory than every baseline.

\subsection{Results on Diffusion Models}
\label{subsec:results_diffusion}

\textbf{Setting.} We fine-tune SD3-Medium on DreamBooth personalization. We select ten subjects fixed prior to training (dog, dog2, cat, backpack, bear plushie, can, candle, colorful sneaker, duck toy, and teapot), covering both live subjects and a diverse range of common object categories, so that conclusions are not tied to a single subject type. A separate adapter is trained per subject, and personalization quality is measured by DINO~\cite{caron2021dino} and CLIP-I~\cite{radford2021clip} for subject fidelity and CLIP-T for text alignment, averaged across subjects and seeds. Following the higher-rank setting studied above, we compare standard LoRA at $r=8$ with CARE-LoRA at $r=16$ to test whether the saved activation memory can also be used effectively in diffusion-transformer fine-tuning.

\textbf{Results.} Table~\ref{tab:diffusion_results} shows that CARE-LoRA at $r=16$ achieves a higher overall average of 46.89, compared to 46.79 for LoRA. Subject fidelity improves on both metrics, with DINO rising from 43.18 to 43.40 and CLIP-I from 69.18 to 69.29, while CLIP-T remains essentially unchanged at 28.00. At the same time, peak memory drops from 9882 to 8630 MiB, a reduction of approximately 13\%. These results confirm that the memory benefit of CARE-LoRA extends to the diffusion-transformer architecture without sacrificing personalization quality.

\begin{table*}[!t]
\centering
\caption{Activation Memory (MiB) of LoRA Layers Under Different Batch Sizes When Fine-Tuning T5-Base on MNLI. 
\textit{Memory Ratio} Is Computed as the Activation Memory of CARE-LoRA Divided by That of LoRA. 
Smaller Values Indicate Lower Activation Memory Cost.}
\label{tab:activation_memory}
\small
\setlength{\tabcolsep}{12pt}
\renewcommand{\arraystretch}{1.05}
\begin{tabular}{l|ccccc}
\toprule
\textbf{Batch Size} 
& 8 
& 16 
& 32 
& 64 
& 128 \\
\midrule
LoRA ($r=8$)
& $528.77$ 
& $1057.55$ 
& $2115.09$ 
& $4230.19$ 
& $8460.38$ \\
CARE-LoRA ($r=16$)
& $\mathbf{24.61}$ 
& $\mathbf{36.84}$ 
& $\mathbf{61.31}$ 
& $\mathbf{110.25}$ 
& $\mathbf{208.13}$ \\
\midrule
Memory Ratio
& $4.65\%$ 
& $3.48\%$ 
& $2.90\%$ 
& $2.61\%$ 
& $2.46\%$ \\
\bottomrule
\end{tabular}
\end{table*}

\subsection{Memory Analysis}
\label{subsec:memory_analysis}

\textbf{Peak memory.} Across all four benchmarks (Tables~\ref{tab:glue_results}, \ref{tab:superglue_results}, \ref{tab:llm_results}, and~\ref{tab:diffusion_results}), CARE-LoRA consistently reduces peak memory relative to standard LoRA. In the T5 experiments, CARE-LoRA at the same rank reduces the total peak memory by roughly $20\%$ relative to LoRA on both GLUE and SuperGLUE. In the larger Mistral-7B-v0.3 and SD3-Medium experiments, CARE-LoRA uses rank $r=16$ and still requires less peak memory than standard LoRA at $r=8$. Since model weights, LoRA parameters, and optimizer states are fixed across methods and occupy a substantial share of total memory, the actual reduction on the activation budget is considerably larger in proportion.

\textbf{Activation memory of LoRA layers.} Table~\ref{tab:activation_memory} further isolates the savings on the LoRA layers, reporting their activation memory when fine-tuning T5-Base on MNLI under varying batch sizes. CARE-LoRA reduces this memory to only 2.5\% to 4.7\% of LoRA's. The ratio decreases as batch size grows, because the compressed activation $Z \in \mathbb{R}^{N \times r}$ scales with $N$ while the rank-$r$ reconstruction matrix $M$ remains fixed, so the per-sample overhead of CARE-LoRA diminishes relative to the full activation $X \in \mathbb{R}^{N \times m}$ retained by LoRA at larger batch sizes.


\begin{table*}[!t]
\centering
\caption{Per-Step Training Time (Seconds) on Fine-Tuning Mistral-7B-v0.3, Averaged Over Three Seeds on Different Training Tasks. 
\textit{Time Ratio} Is Computed as Each Method's Average Per-Step Time Divided by That of LoRA.}
\label{tab:llm_time}
\small
\setlength{\tabcolsep}{10pt}
\renewcommand{\arraystretch}{1.05}
\begin{tabular}{l|ccc|cc}
\toprule
\textbf{Method} 
& \textbf{Math} 
& \textbf{Code} 
& \textbf{Instruct} 
& \textbf{Avg. Time} 
& \textbf{Time Ratio} \\
\midrule
LoRA ($r=8$)
& $2.26$ 
& $4.66$ 
& $11.32$ 
& $6.08$ 
& $1.00\times$ \\
\midrule
DoRA ($r=8$)
& $7.49$ 
& $15.29$ 
& $38.91$ 
& $20.56$ 
& $3.39\times$ \\
LoRAct ($r=8$)
& $3.25$ 
& $6.42$ 
& $15.48$ 
& $8.38$ 
& $1.38\times$ \\
LoRA-FA ($r=8$)
& $\mathbf{2.24}$ 
& $\mathbf{4.32}$ 
& $\mathbf{10.83}$ 
& $\mathbf{5.80}$ 
& $\mathbf{0.95}\times$ \\
CARE-LoRA ($r=16$)
& $\underline{2.34}$ 
& $\underline{4.74}$ 
& $\underline{11.59}$ 
& $\underline{6.22}$ 
& $\underline{1.02}\times$ \\
\bottomrule
\end{tabular}
\end{table*}


\subsection{Time Analysis}
\label{subsec:time_analysis}

\begin{table}[!t]
\centering
\caption{Comparison With Gradient Checkpointing Under a Similar Memory Budget When Fine-Tuning Mistral-7B-v0.3 on the Math Task. 
Per-Step Training Time (Seconds) and Peak CUDA Memory (MiB) Are Averaged Over Three Seeds.}
\label{tab:gradckpt_time}
\small
\setlength{\tabcolsep}{5pt}
\renewcommand{\arraystretch}{1.05}
\begin{tabular}{l|cc}
\toprule
\textbf{Method} 
& \textbf{Time / Step} 
& \textbf{Peak Mem} \\
\midrule
LoRA + GC ($r=8$)
& $2.46$ 
& $25992.36$ \\
CARE-LoRA ($r=16$)
& $\mathbf{2.34}$ 
& $\mathbf{25272.29}$ \\
\bottomrule
\end{tabular}
\end{table}

\textbf{Per-step training time.} Table~\ref{tab:llm_time} compares the average per-step training time on Mistral-7B-v0.3, normalized by LoRA. We omit PiSSA, whose per-step computation matches LoRA once initialized. For CARE-LoRA, we report the costlier $r=16$ configuration as a conservative case. DoRA combines per-step weight renormalization with backward-pass recomputation in its memory-saving implementation, making it $3.39\times$ slower than LoRA despite its competitive accuracy. LoRAct performs an online low-rank decomposition of the cached activation, raising the per-step time to $1.38\times$ that of LoRA. LoRA-FA is the fastest but least accurate baseline. CARE-LoRA at $r=16$ requires only $1.02\times$ the per-step time of LoRA at $r=8$, nearly matching it. Combined with the accuracy and memory results above, these findings show that CARE-LoRA is the only method that is simultaneously memory-efficient, fast, and accurate, whereas each baseline sacrifices at least one of these three axes.

\textbf{Comparison with gradient checkpointing.} Gradient checkpointing (GC)~\cite{chen2016training} is a standard approach to activation-memory reduction, trading compute for memory by recomputing activations during the backward pass. To enable a direct comparison, we apply GC to the first 25\% of Mistral-7B-v0.3's decoder blocks, i.e., 8 of 32, during LoRA fine-tuning on the Math task, so that its memory budget roughly matches that of CARE-LoRA at $r=16$. Since GC does not alter the optimization, LoRA+GC retains the accuracy of LoRA, so we focus on memory and time. As shown in Table~\ref{tab:gradckpt_time}, under a comparable memory budget, CARE-LoRA at $r=16$ outperforms LoRA+GC in both per-step time (2.34 vs.\ 2.46 s) and peak memory (25272 vs.\ 25992 MiB). The advantage is structural: GC saves memory by recomputing entire blocks, whereas CARE-LoRA exploits the inherent low-rank structure of LoRA, substantially reducing compute overhead at the same memory level.

\section{Conclusion}
\label{sec:conclusion}

In this paper, we propose CARE-LoRA, a data-aware framework that replaces the full activation retained by standard LoRA with a compressed activation and a lightweight reconstruction matrix, keeping both LoRA matrices trainable at a memory cost close to LoRA-FA. Building on this reconstruction, we proved that the projection-down subspace is not frozen: under non-degenerate activations and backward signals, it can evolve during training, unlike the fixed subspace that LoRA-FA imposes. Extensive experiments across diverse tasks show that CARE-LoRA closes most of the accuracy gap LoRA-FA leaves relative to standard LoRA, and that reinvesting the saved memory into a higher rank lets it surpass standard LoRA and representative variants with only a negligible increase in training time, achieving a favorable balance of memory efficiency, speed, and accuracy.

\bibliographystyle{IEEEtran}
\bibliography{references}

\newpage

\appendices
\numberwithin{equation}{section}
\renewcommand{\theequation}{\thesection\arabic{equation}}

\section{Proofs for Subspace Adaptation}
\label{app:subspace_proofs}

This appendix gives the detailed derivations for the subspace adaptation analysis in Section~III-C. We first recall, with explicit intermediate steps, how the reconstructed gradient of $A$ yields a condition for the projection-down subspace to change. We then give complete proofs of two claims that Section~III-C states without proof, namely that the data-dependent part is nonzero except on a measure-zero set of initializations, and that the matrix $Z^\top U$ governing the gradient-dependent part is generically nonsingular.

\subsection{Gradient Form and the Movement Condition}
\label{app:proof_covariance_form}

Recall that
\begin{equation*}
    Z = XA,\qquad
    U = G B^\top,\qquad
    \Sigma_X = X^\top X.
\end{equation*}
CARE-LoRA stores the regularized reconstruction matrix
\begin{equation}
    M = (Z^\top Z+\lambda I_r)^{-1}Z^\top X.
\end{equation}
The reconstructed gradient for $A$ is
\begin{equation}
    \widetilde{\nabla}_A\mathcal{L}
    =
    M^\top Z^\top U.
\end{equation}
Taking the transpose of $M$ gives
\begin{equation}
    M^\top
    =
    X^\top Z (Z^\top Z+\lambda I_r)^{-1}.
\end{equation}
Therefore,
\begin{align}
    \widetilde{\nabla}_A\mathcal{L}
    &=
    X^\top Z (Z^\top Z+\lambda I_r)^{-1} Z^\top U \notag \\
    &=
    X^\top X A
    (A^\top X^\top X A+\lambda I_r)^{-1}
    A^\top X^\top U \notag \\
    &=
    \Sigma_X A
    (A^\top \Sigma_X A+\lambda I_r)^{-1}
    A^\top X^\top U.
\end{align}
Defining
\begin{equation}
    K_\lambda =
    (A^\top \Sigma_X A+\lambda I_r)^{-1}
    A^\top X^\top U
\end{equation}
yields
\begin{equation}
    \widetilde{\nabla}_A\mathcal{L}
    =
    \Sigma_X A K_\lambda.
\end{equation}
Because $\lambda>0$ and $A^\top\Sigma_X A$ is positive semidefinite, $A^\top\Sigma_X A+\lambda I_r$ is positive definite and therefore invertible. Thus, the regularized form is always well-defined.

We now use this form to determine when a gradient step moves $A$ outside its current column space. Let
\begin{equation*}
    S=\operatorname{Col}(A),
    \qquad
    P_S^\perp=I_m-P_S.
\end{equation*}
Since every column of $A$ lies in $S$, we have
\begin{equation}
    P_S^\perp A = 0.
\end{equation}
For a gradient step
\begin{equation}
    A^+ = A - \eta\Sigma_X A K_\lambda,
\end{equation}
left-multiplying by $P_S^\perp$ gives
\begin{align}
    P_S^\perp A^+
    &=
    P_S^\perp A
    -
    \eta P_S^\perp\Sigma_X A K_\lambda \notag \\
    &=
    -\eta P_S^\perp\Sigma_X A K_\lambda.
\end{align}
Therefore, if
\begin{equation}
    P_S^\perp\Sigma_X A K_\lambda\neq 0,
\end{equation}
then $P_S^\perp A^+\neq 0$. Hence at least one column of $A^+$ has a nonzero component in $S^\perp$, which implies
\begin{equation}
    \operatorname{Col}(A^+)\nsubseteq S
    =
    \operatorname{Col}(A).
\end{equation}
If $A^+$ has full column rank, then $\operatorname{Col}(A^+)$ is again an $r$-dimensional subspace, and the learned projection-down subspace has changed. As stated in Section~III-C, this condition separates into a data-dependent part, $P_S^\perp\Sigma_X A\neq0$, and a gradient-dependent part concerning whether $K_\lambda$ cancels it. The next two subsections prove that each part holds except on a measure-zero set of degenerate cases.

\subsection{Invariance Analysis of the Data-Dependent Part}
\label{app:proof_invariant_subspace}

Section~III-C states that the data-dependent condition $P_S^\perp\Sigma_X A\neq0$ fails exactly when $\operatorname{Col}(A)$ is an invariant subspace of $\Sigma_X$, but does not prove this equivalence. We prove it here, and then justify why such an invariant subspace is nongeneric.

We show that
\begin{equation}
    P_S^\perp\Sigma_X A = 0
\end{equation}
is equivalent to $S=\operatorname{Col}(A)$ being invariant under $\Sigma_X$.

First, if $P_S^\perp\Sigma_X A=0$, then every column of $\Sigma_X A$ lies in $S$, so
\begin{equation}
    \operatorname{Col}(\Sigma_X A)\subseteq S.
\end{equation}
For any vector $s\in S$, there exists $v\in\mathbb{R}^r$ such that $s=Av$. Then
\begin{equation}
    \Sigma_X s = \Sigma_X A v \in \operatorname{Col}(\Sigma_X A)\subseteq S.
\end{equation}
Thus $\Sigma_X S\subseteq S$, i.e., $S$ is an invariant subspace of $\Sigma_X$.

Conversely, if $\Sigma_X S\subseteq S$, then each column of $A$ belongs to $S$, and therefore each column of $\Sigma_X A$ also belongs to $S$. Hence
\begin{equation}
    P_S^\perp\Sigma_X A=0.
\end{equation}
This proves
\begin{equation}
    P_S^\perp\Sigma_X A=0
    \quad\Longleftrightarrow\quad
    \Sigma_X S\subseteq S.
\end{equation}

It remains to argue that a randomly initialized $\operatorname{Col}(A)$ is essentially never an invariant subspace of $\Sigma_X$. For a fixed non-scalar symmetric matrix $\Sigma_X$, a continuously sampled $r$-dimensional subspace is not an invariant subspace with probability one. Indeed, the invariant subspaces of a symmetric matrix are constrained by its eigenspace decomposition. Every invariant subspace decomposes as a direct sum of subspaces drawn from the individual eigenspaces of $\Sigma_X$, so an invariant $r$-dimensional subspace must align with this eigenstructure rather than sit in a generic position. Unless $\Sigma_X$ is proportional to the identity, this alignment requirement confines the invariant $r$-dimensional subspaces to a measure-zero subset of the Grassmannian of all $r$-dimensional subspaces, so a continuously distributed initialization of $\operatorname{Col}(A)$ avoids them almost surely. The isotropic case $\Sigma_X=cI_m$ is the degenerate exception, since every subspace is invariant.

\subsection{Genericity of the Gradient-Dependent Part}
\label{app:generic_ztu}

We now turn to the gradient-dependent part. Let
\begin{equation}
    D = P_S^\perp\Sigma_X A.
\end{equation}
Even when $D\neq 0$, the outside-subspace component of the CARE-LoRA gradient is
\begin{equation*}
    D K_\lambda.
\end{equation*}
Thus the outside component vanishes if and only if every column of $K_\lambda$ lies in the null space of $D$
\begin{equation}
    D K_\lambda = 0
    \quad\Longleftrightarrow\quad
    \operatorname{Col}(K_\lambda)\subseteq \operatorname{Null}(D).
\end{equation}
This is a backward-signal degeneracy. Since
\begin{align}
    K_\lambda
    &=
    (A^\top\Sigma_X A+\lambda I_r)^{-1}A^\top X^\top U \notag \\
    &=
    (A^\top\Sigma_X A+\lambda I_r)^{-1}Z^\top U,
\end{align}
and the first factor is invertible for $\lambda>0$, this degeneracy is governed entirely by the rank of $Z^\top U$. Because an invertible factor preserves rank, $K_\lambda$ is singular if and only if $Z^\top U$ is. For example, if $U=GB^\top=0$, then $Z^\top U=0$, so $K_\lambda=0$ and no update to $A$ is produced. This happens at the first optimization step under the common LoRA initialization $B=0$.

Section~III-C asserts that, away from such special cases, $Z^\top U$ is generically nonsingular, without proof. We prove this claim below.

\begin{lemma}[Generic nonsingularity of $Z^\top U$]
Let $N\ge r$ and let $Z,U\in\mathbb{R}^{N\times r}$. Suppose that the joint distribution of $(Z,U)$ is absolutely continuous with respect to the Lebesgue measure on $\mathbb{R}^{N\times r}\times\mathbb{R}^{N\times r}$. Then
\begin{equation}
    \Pr\!\left(\operatorname{rank}(Z^\top U)<r\right)=0 .
\end{equation}
In particular, since $\operatorname{rank}(Z)=r$ and $\operatorname{rank}(U)=r$ also hold with probability one, we have
\begin{equation}
    \Pr\!\left(\operatorname{rank}(Z^\top U)<r
    \,\middle|\,
    \operatorname{rank}(Z)=\operatorname{rank}(U)=r
    \right)=0 .
\end{equation}
\end{lemma}

\begin{proof}
The condition $\operatorname{rank}(Z^\top U)<r$ is equivalent to
\begin{equation}
    \det(Z^\top U)=0 .
\end{equation}
The function
\begin{equation}
    p(Z,U)=\det(Z^\top U)
\end{equation}
is a polynomial in the entries of $Z$ and $U$. This polynomial is not identically zero. For example, taking
\begin{equation}
    Z=
    \begin{bmatrix}
        I_r\\
        0
    \end{bmatrix},
    \qquad
    U=
    \begin{bmatrix}
        I_r\\
        0
    \end{bmatrix}
\end{equation}
gives $Z^\top U=I_r$ and hence $p(Z,U)=1$. Therefore, the zero set $\{(Z,U):p(Z,U)=0\}$ is a proper algebraic variety and has Lebesgue measure zero. Since $(Z,U)$ has an absolutely continuous joint distribution, the probability of this event is zero.

The same argument also implies that $\operatorname{rank}(Z)=r$ and $\operatorname{rank}(U)=r$ hold with probability one, because the vanishing of all $r\times r$ minors is again a measure-zero algebraic condition. This proves the conditional statement.
\end{proof}

Together with Section~\ref{app:proof_invariant_subspace}, this shows that, under the assumptions stated above, the set of initializations and backward signals for which $P_S^\perp\Sigma_X A K_\lambda=0$ has measure zero. Hence, the projection-down subspace generically evolves under CARE-LoRA.

\section{Detailed Memory Analysis}
\label{app:memory_details}

This appendix gives the detailed derivation for the memory analysis of the main paper. We analyze one adapted linear layer
\begin{equation*}
\begin{aligned}
    Y &= X(W+AB), \\
    X &\in\mathbb{R}^{N\times m},\quad
    A\in\mathbb{R}^{m\times r},\quad 
    B \in\mathbb{R}^{r\times n}.
\end{aligned}
\end{equation*}
The count below only includes tensors specific to the LoRA branch. The frozen backbone path $XW$ and other model activations are shared by the compared methods.

\subsection{Activation Memory}
\label{app:activation_memory}

Standard LoRA keeps $A$ and $B$ trainable. To compute
\begin{equation}
    \nabla_A\mathcal{L}=X^\top G B^\top,
\end{equation}
the full activation $X\in\mathbb{R}^{N\times m}$ must be available in the backward pass. Following the dominant-term accounting used in the main text, the LoRA-side activation count is
\begin{equation}
    \mathcal{M}_{\mathrm{LoRA}}^{\mathrm{act}}
    =
    Nm.
\end{equation}

LoRA-FA freezes $A$ and only updates $B$. Since it does not need $\nabla_A\mathcal{L}$, it only saves
\begin{equation}
    Z=XA\in\mathbb{R}^{N\times r},
\end{equation}
and therefore
\begin{equation}
    \mathcal{M}_{\mathrm{LoRA\text{-}FA}}^{\mathrm{act}}
    =
    Nr.
\end{equation}

CARE-LoRA saves $Z$ and the batch-dependent reconstruction matrix $M\in\mathbb{R}^{r\times m}$, giving
\begin{equation}
    \mathcal{M}_{\mathrm{CARE}}^{\mathrm{act}}
    =
    Nr+rm.
\end{equation}
It is smaller than the LoRA-side activation count whenever
\begin{equation}
    Nr+rm < Nm
    \quad\Longleftrightarrow\quad
    r < \frac{Nm}{N+m}.
\end{equation}
For example, when $N=8192$ and $m=768$, the right-hand side is approximately $702.2$, which is far larger than common LoRA ranks such as $8$ or $16$, so this condition is easily satisfied in typical fine-tuning settings.

\subsection{Layer-Wise Numerical Examples}
\label{app:memory_examples}

For $N=8192$, the activation counts for a single adapted layer are as follows.

\begin{table*}[!t]
\centering
\caption{LoRA-Side Activation Storage for One Adapted Linear Layer. Counts Are Reported in Number of Scalar Values.}
\label{tab:activation_storage_examples}
\setlength{\tabcolsep}{12pt}
\small
\renewcommand{\arraystretch}{1.05}
\begin{tabular}{ccccc}
\toprule
Input dim. $m$ & Rank $r$ & LoRA $Nm$ & CARE $r(N+m)$ & CARE / LoRA \\
\midrule
$768$  & $8$  & $6.29$M  & $71.7$K  & $1.14\%$ \\
$768$  & $16$ & $6.29$M  & $143.4$K  & $2.28\%$ \\
$3072$ & $8$  & $25.17$M & $90.1$K  & $0.36\%$ \\
$3072$ & $16$ & $25.17$M & $180.2$K & $0.72\%$ \\
\bottomrule
\end{tabular}
\end{table*}

If LoRA is applied to all attention and feed-forward linear layers of a T5-Base-style encoder-decoder model, the same calculation can be aggregated over layers. With $12$ encoder blocks and $12$ decoder blocks, an encoder block contains four attention projections and two feed-forward projections, while a decoder block contains self-attention, cross-attention, and two feed-forward projections. Using $d_{\mathrm{model}}=768$ and $d_{\mathrm{ff}}=3072$, the sum of input dimensions over these adapted linear layers is
\begin{equation*}
\begin{aligned}
    \sum_{\ell} m_\ell
    &=
    12(5\cdot 768 + 3072)
    +
    12(9\cdot 768 + 3072)
    \\
    &=
    202{,}752,
\end{aligned}
\end{equation*}
and the number of adapted linear layers is
\begin{equation*}
    L_{\mathrm{adapt}} = 12\cdot 6 + 12\cdot 10 = 192.
\end{equation*}
For $N=8192$, standard LoRA stores
\begin{equation*}
    N\sum_{\ell}m_\ell
    =
    1660.9\mathrm{M}
\end{equation*}
LoRA-side activation values, independent of $r$ under this dominant-term count. For $r=8$, CARE-LoRA stores
\begin{equation*}
    r\left(NL_{\mathrm{adapt}}+\sum_{\ell}m_\ell\right)
    =
    14.20\mathrm{M}
\end{equation*}
values, about $0.86\%$ of the standard LoRA count. Assuming two bytes per value in BF16 or FP16, this corresponds to roughly $3.09$GiB versus $27.1$MiB for these adapter-side activation buffers. For $r=16$, the corresponding CARE-LoRA count is $28.41\mathrm{M}$ values, or about $1.71\%$ of standard LoRA.

\section{Detailed Computation Analysis}
\label{app:compute_details}

This appendix gives the detailed derivation for the computation analysis of the main paper, using the same adapted linear layer as in Appendix~\ref{app:memory_details}, with input activation $X\in\mathbb{R}^{N\times m}$, projection-down matrix $A\in\mathbb{R}^{m\times r}$, and projection-up matrix $B\in\mathbb{R}^{r\times n}$. As in Appendix~\ref{app:memory_details}, the count below only includes matrix multiplications specific to the LoRA branch, excluding the shared frozen backbone path.

\subsection{Computation Cost}
\label{app:compute_cost}

We count matrix multiplication costs up to constant factors, i.e., using multiply-accumulate order rather than distinguishing a factor of two in FLOPs. The LoRA branch consists of the following operations.

\textbf{Standard LoRA.}
In the forward pass, standard LoRA computes
\begin{equation*}
    Z=XA,
    \qquad
    \Delta Y=ZB,
\end{equation*}
with costs $Nmr$ and $Nrn$, respectively.

In the backward pass, let $G=\partial\mathcal{L}/\partial Y$. Standard LoRA computes
\begin{equation}
    U=GB^\top\in\mathbb{R}^{N\times r},
\end{equation}
with cost $Nrn$, then
\begin{equation}
    \nabla_A\mathcal{L}=X^\top U
\end{equation}
with cost $Nmr$, and
\begin{equation}
    \nabla_B\mathcal{L}=Z^\top G
\end{equation}
with cost $Nrn$. The LoRA-side input-gradient contribution
\begin{equation}
    UA^\top
\end{equation}
has cost $Nmr$. Therefore,
\begin{equation}
    T_{\mathrm{LoRA}}
    =
    \underbrace{Nmr}_{XA}
    +
    \underbrace{Nrn}_{ZB}
    +
    \underbrace{Nrn}_{GB^\top}
    +
    \underbrace{Nmr}_{X^\top U}
    +
    \underbrace{Nrn}_{Z^\top G}
    +
    \underbrace{Nmr}_{UA^\top},
\end{equation}
which gives
\begin{equation}
    T_{\mathrm{LoRA}}
    =
    3Nmr+3Nrn.
\end{equation}

\textbf{CARE-LoRA.}
CARE-LoRA shares the same $XA$, $ZB$, $GB^\top$, $Z^\top G$, and $UA^\top$ computations. The difference is how the $A$-gradient is obtained.

In the forward pass, CARE-LoRA forms
\begin{equation*}
    H=Z^\top Z+\lambda I_r,
    \qquad
    C=Z^\top X,
\end{equation*}
and solves
\begin{equation}
    M = \mathrm{solve}(H,C).
\end{equation}
The cost of forming $H$ is $Nr^2$. The cost of forming $C$ is $Nmr$. This $Nmr$ term replaces the $Nmr$ term used by standard LoRA to compute $X^\top U$ in the backward pass. Therefore, it is not an additional large-matrix term relative to standard LoRA. Since $H$ is an $r\times r$ symmetric positive definite matrix for $\lambda>0$, it can be factorized by Cholesky decomposition and solved against $m$ right-hand sides. This contributes
\begin{equation}
    O(r^3+mr^2).
\end{equation}

In the backward pass, CARE-LoRA computes
\begin{equation}
    T=Z^\top U
\end{equation}
with cost $Nr^2$, and then
\begin{equation}
    \widetilde{\nabla}_A\mathcal{L}=M^\top T
\end{equation}
with cost $mr^2$. Therefore,
\begin{equation}
\begin{aligned}
    T_{\mathrm{CARE}}
    &=
    3Nmr+3Nrn
    +
    2Nr^2+2mr^2+r^3 .
\end{aligned}
\end{equation}
The ratio between CARE-LoRA and standard LoRA is
\begin{equation}
\begin{aligned}
    \frac{T_{\mathrm{CARE}}}{T_{\mathrm{LoRA}}}
    &=
    \frac{
    3Nmr+3Nrn+2Nr^2+2mr^2+r^3
    }{
    3Nmr+3Nrn
    }
    \\
    &=
    1+
    \frac{2Nr^2+2mr^2+r^3}{3Nr(m+n)}
    \\
    &=
    1+
    \frac{r(2N+2m+r)}{3N(m+n)}.
\end{aligned}
\end{equation}
Thus the extra arithmetic scales with $r$ relative to the large feature dimensions. For common settings with $r\ll\min(m,n,N)$, the additional algebraic cost is small.

\subsection{Numerical Cost Ratios}
\label{app:compute_examples}

For a hidden-to-hidden layer with $N=8192$, $m=n=768$, and $r=8$,
\begin{equation*}
    T_{\mathrm{LoRA}}
    =
    3\cdot8192\cdot768\cdot8
    +
    3\cdot8192\cdot8\cdot768
    =
    301.99\mathrm{M},
\end{equation*}
and the reconstruction overhead is
\begin{equation*}
    2\cdot8192\cdot8^2
    +
    2\cdot768\cdot8^2
    +
    8^3
    =
    1.147\mathrm{M}.
\end{equation*}
Therefore,
\begin{equation*}
    \frac{T_{\mathrm{CARE}}}{T_{\mathrm{LoRA}}}
    \approx
    1.0038.
\end{equation*}
For $r=16$, the ratio becomes approximately $1.0076$.

For a feed-forward down-projection input with $m=3072$, $n=768$, and $r=8$, the same formula gives
\begin{equation*}
    \frac{T_{\mathrm{CARE}}}{T_{\mathrm{LoRA}}}
    \approx
    1.0019.
\end{equation*}
For $r=16$, the ratio is approximately $1.0038$. These ratios should be understood as algebraic operation counts. Actual wall-clock time also depends on kernel fusion, memory access patterns, and the efficiency of the small Cholesky solve.

\section{Directional Fidelity of the Reconstructed Gradient}
\label{app:grad_fidelity}

\textbf{Global Cosine Similarity.} CARE-LoRA updates the projection-down matrix with the reconstructed gradient $\widetilde{\nabla}_A\mathcal{L}=\widehat{X}^\top G B^\top$ in place of the exact gradient $\nabla_A\mathcal{L}=X^\top G B^\top$. We quantify the fidelity of this approximation via the global cosine similarity between the exact and reconstructed $A$-gradients across all $L$ adapted layers:
\begin{equation}
C_{\mathrm{global}}
= \frac{\langle g_{\mathrm{all}},\widetilde g_{\mathrm{all}}\rangle}{\|g_{\mathrm{all}}\|_2\,\|\widetilde g_{\mathrm{all}}\|_2},
\label{eq:global_cosine_appendix}
\end{equation}
where $g_{\mathrm{all}}=\big[\operatorname{vec}(\nabla_{A_1}\mathcal{L})^\top,\dots,\operatorname{vec}(\nabla_{A_L}\mathcal{L})^\top\big]^\top$ and $\widetilde g_{\mathrm{all}}$ is defined analogously. By construction $C_{\mathrm{global}}\in[-1,1]$, with $C_{\mathrm{global}}=1$ indicating perfect directional alignment between the two gradients. In practice, $g_{\mathrm{all}}$ need not be explicitly constructed because the inner product and squared norms can be computed and accumulated layer by layer. We focus on direction rather than magnitude because differences in magnitude can be absorbed by the learning rate and further attenuated by AdamW's per-coordinate rescaling.

\textbf{Setting.} We measure $C_{\mathrm{global}}$ along a standard CARE-LoRA fine-tuning run of Mistral-7B-v0.3 on the Math task with $r=16$ and seed 0. The trajectory is driven entirely by the reconstructed gradient, and at diagnostic checkpoints we temporarily retain $X$ to compute the exact $\nabla_A\mathcal{L}$ for measurement only. These checkpoints are placed evenly at $5\%, 10\%, \dots, 100\%$ of the optimizer steps. The $0\%$ point is excluded because the default initialization $B=0$ forces both gradients to vanish and leaves the cosine undefined.

\textbf{Results.} Figure~\ref{fig:grad_cosine_appendix} plots $C_{\mathrm{global}}$ along the training trajectory. $C_{\mathrm{global}}$ stays consistently above $0.6$ from the first valid checkpoint to the last, with no visible degradation across training. The reconstructed $A$-gradient therefore retains strong directional agreement with the exact $\nabla_A\mathcal{L}$ throughout fine-tuning, justifying its use as the surrogate update direction in CARE-LoRA.

\begin{figure}[!t]
\centering
\includegraphics[width=\columnwidth]{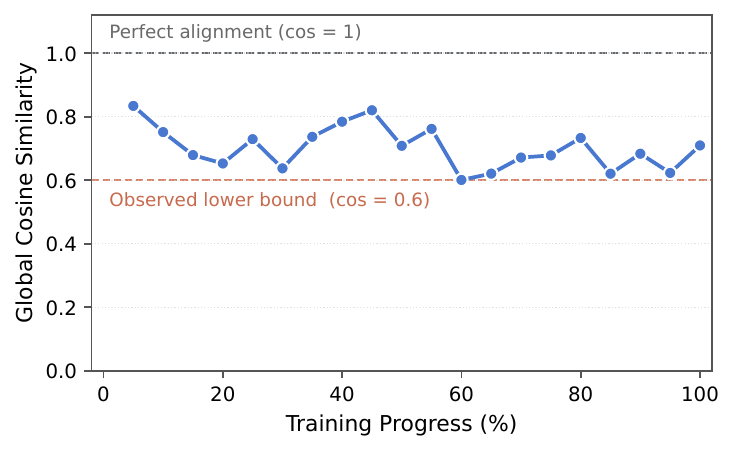}
\caption{Global cosine similarity between the exact and reconstructed $A$-gradients of CARE-LoRA along the training trajectory on the Math task. The dashed line at $1.0$ marks perfect alignment.}
\label{fig:grad_cosine_appendix}
\end{figure}

\textbf{Geometric interpretation.} The observed cosine must be read against the geometry of the ambient space. The concatenated $A$-gradient vector in Eq.~\eqref{eq:global_cosine_appendix} lives in $\mathbb{R}^{d}$ with $d=r\sum_{\ell}m_\ell$, where $m_\ell$ is the input width of the $\ell$-th adapted layer. For Mistral-7B-v0.3 with $r=16$ and $L=224$ adapted linear layers, $\sum_\ell m_\ell=1{,}245{,}184$, giving $d\approx 1.99\times 10^{7}$. In a space of this dimension, two directions chosen without relation to each other are almost exactly orthogonal, as the following lemma makes precise.

\begin{lemma}[Near-orthogonality in high dimension~\cite{vershynin2018hdp}]
\label{lem:near_orthogonality_appendix}
Let $u,v$ be independent and uniformly distributed on the unit sphere $\mathbb{S}^{d-1}\subset\mathbb{R}^d$. Then $\mathbb{E}\langle u,v\rangle=0$, $\operatorname{Var}\langle u,v\rangle=1/d$, and for every $\varepsilon\in(0,1)$,
\begin{equation}
\Pr\!\big[\,|\langle u,v\rangle|\ge\varepsilon\,\big]\;\le\;2\exp\!\Big(\!-\tfrac{d\varepsilon^2}{2}\Big).
\label{eq:cap_bound_appendix}
\end{equation}
\end{lemma}

Instantiated here, Lemma~\ref{lem:near_orthogonality_appendix} implies that, under the random-direction null model, an
uninformative reconstruction would produce a cosine of typical magnitude $1/\sqrt{d}\approx2.2\times10^{-4}$, and that the
probability of reaching even $0.6$ by chance is at most $2\exp(-3.6\times10^{6})$, effectively zero. This provides
supporting evidence that the reconstructed gradient retains meaningful directional alignment with the exact $A$-gradient.
We further note that $0.6$ is an empirical lower bound across our diagnostic checkpoints rather than a typical value, so the
observed directional alignment over the bulk of training is generally stronger than this conservative reading suggests.

\begin{table}[!t]
\centering
\small
\setlength{\tabcolsep}{4pt}
\renewcommand{\arraystretch}{1.05}
\caption{Hyperparameter Configuration for Fine-Tuning T5-Base on the GLUE Datasets.}
\label{tab:app_glue_hparams}
\begin{tabular}{lc}
\toprule
\textbf{Hyperparameter} & \textbf{Value} \\
\midrule
Learning Rate & $5\times10^{-4}$ \\
Batch Size & $32$ \\
Epochs (MNLI, QNLI) & $1$ \\
Epochs (others) & $10$ \\
Max Sequence Length & $256$ \\
LoRA Rank $r$ & $8$ \\
LoRA Alpha $\alpha$ & $16$ \\
Target Modules & All linear layers \\
LR Scheduler & Cosine \\
Warmup Ratio & $0.03$ \\
Seed & $0$, $21$, $42$ \\
\midrule
Evaluation Metric (CoLA) & Matthews Corr. \\
Evaluation Metric (MRPC) & F1 \\
Evaluation Metric (others) & Accuracy \\
\bottomrule
\end{tabular}
\end{table}

\begin{table}[!t]
\centering
\small
\setlength{\tabcolsep}{4pt}
\renewcommand{\arraystretch}{1.05}
\caption{Hyperparameter Configuration for Fine-Tuning T5-Base on the SuperGLUE Datasets.}
\label{tab:app_superglue_hparams}
\begin{tabular}{lc}
\toprule
\textbf{Hyperparameter} & \textbf{Value} \\
\midrule
Learning Rate & $5\times10^{-4}$ \\
Batch Size & $32$ \\
Epochs (BoolQ) & $20$ \\
Epochs (CB, COPA) & $50$ \\
Epochs (RTE, WiC) & $10$ \\
Max Sequence Length & $256$ \\
LoRA Rank $r$ & $8$ \\
LoRA Alpha $\alpha$ & $16$ \\
Target Modules & All linear layers \\
LR Scheduler & Cosine \\
Warmup Ratio & $0.03$ \\
Seed & $0$, $21$, $42$ \\
\midrule
Evaluation Metric (CB) & F1 \\
Evaluation Metric (others) & Accuracy \\
\bottomrule
\end{tabular}
\end{table}

\section{Detailed Experimental Configurations}
\label{app:exp_config}

This appendix reports the key hyperparameters, evaluation protocols, and hardware used for the three experimental settings in Section~IV, covering T5-Base fine-tuning on GLUE and SuperGLUE, Mistral-7B-v0.3 fine-tuning on mathematical reasoning, code generation, and instruction following, and SD3-Medium fine-tuning on DreamBooth personalization.

Unless otherwise noted, all methods across the three settings use the standard LoRA initialization, with $A$ drawn from a Kaiming uniform distribution and $B$ set to zero, the default and most widely used initialization for LoRA-based fine-tuning. The exception is PiSSA, which instead uses the SVD-based initialization already described in Section~IV-C.

For CARE-LoRA, we use a fixed diagonal offset $\lambda=10^{-6}$ in all experiments when solving $(Z^\top Z+\lambda I_r)M=Z^\top X$. This value is used for numerical stability and is not tuned across tasks.

\subsection{Experiments on Natural Language Understanding}
\label{app:config_nlu}

Table~\ref{tab:app_glue_hparams} and Table~\ref{tab:app_superglue_hparams} list the hyperparameters used for fine-tuning T5-Base on GLUE and SuperGLUE, corresponding to the same-rank results in Section~IV-B. All compared methods use rank $r=8$ with $\alpha=16$.

The number of training epochs varies by task, as shown in Table~\ref{tab:app_glue_hparams} and Table~\ref{tab:app_superglue_hparams}. We evaluate on the validation set after every epoch and report the best-epoch result for each task.

All experiments in this subsection are conducted on a single NVIDIA GeForce RTX 2080 Ti 11GB.

\subsection{Experiments on Large Language Models}
\label{app:config_llm}

Table~\ref{tab:app_llm_hparams} lists the hyperparameters used for fine-tuning Mistral-7B-v0.3 on mathematical reasoning, code generation, and instruction following, corresponding to the results in Section~IV-C. The main comparison uses rank $r=8$ with $\alpha=16$ for all methods, and CARE-LoRA is additionally evaluated at $r=16$ with $\alpha=32$ to study whether the saved activation memory can be reinvested into stronger performance, keeping $\alpha/r=2$ fixed.

\begin{table}[!t]
\centering
\small
\setlength{\tabcolsep}{2pt}
\renewcommand{\arraystretch}{1.05}
\caption{Hyperparameter Configuration for Fine-Tuning Mistral-7B-v0.3 on Mathematical Reasoning, Code Generation, and Instruction Following.}
\label{tab:app_llm_hparams}
\begin{tabular}{lc}
\toprule
\textbf{Hyperparameter} & \textbf{Value} \\
\midrule
Learning Rate (Math, Instruct) & $5\times10^{-5}$ \\
Learning Rate (Code) & $1\times10^{-4}$ \\
Batch Size & $32$ \\
Micro-Batch Size (Math) & $4$ \\
Micro-Batch Size (Code) & $2$ \\
Micro-Batch Size (Instruct) & $1$ \\
Epochs & $1$ \\
Max Sequence Length (Math) & $512$ \\
Max Sequence Length (Code) & $1024$ \\
Max Sequence Length (Instruct) & $2048$ \\
LoRA Rank $r$ & 8 (main), 16 (CARE-LoRA) \\
LoRA Alpha $\alpha$ & 16 (main), 32 (CARE-LoRA) \\
Target Modules & All linear layers \\
LR Scheduler & Cosine \\
Warmup Ratio & $0.03$ \\
Seed & $0$, $21$, $42$ \\
\midrule
Evaluation Metric (Math) & Accuracy \\
Evaluation Metric (Code) & Pass@1 \\
Evaluation Metric (Instruct) & Prompt/Instruction-Strict \\
\bottomrule
\end{tabular}
\end{table}

\begin{table}[!t]
\centering
\small
\setlength{\tabcolsep}{4pt}
\renewcommand{\arraystretch}{1.05}
\caption{Training and Generation Configuration for Fine-Tuning SD3-Medium on DreamBooth Personalization.}
\label{tab:app_diffusion_hparams}
\begin{tabular}{lc}
\toprule
\textbf{Hyperparameter} & \textbf{Value} \\
\midrule
Resolution & $512\times512$ \\
Batch Size & $4$ \\
Training Steps & $500$ per subject \\
Learning Rate & $2\times10^{-4}$ \\
LR Scheduler & Constant \\
Warmup Steps & $0$ \\
LoRA Rank $r$ & $8$ (LoRA), $16$ (CARE-LoRA) \\
LoRA Alpha $\alpha$ & $16$ (LoRA), $32$ (CARE-LoRA) \\
Target Modules & SD3 attention projections \\
Seed & $0$, $21$, $42$ \\
\midrule
Inference Steps & $28$ \\
Guidance Scale & $7.0$ \\
Images per Prompt & $4$ \\
Generation Seed & $42$ \\
Evaluation Metrics & DINO, CLIP-I, CLIP-T \\
\bottomrule
\end{tabular}
\end{table}

All experiments in this subsection are conducted on a single NVIDIA GeForce RTX 5090 32GB.

\subsection{Experiments on Diffusion Models}
\label{app:config_diffusion}

Table~\ref{tab:app_diffusion_hparams} lists the training and generation configuration used for fine-tuning SD3-Medium on DreamBooth personalization, corresponding to the results in Section~IV-D. Each of the ten subjects is trained independently under this configuration for both standard LoRA and CARE-LoRA.

All experiments in this subsection are conducted on a single NVIDIA GeForce RTX 4090 24GB.

\section{Broader Impacts and Limitations}
\label{app:impacts_limitations}

\textbf{Broader Impacts.} CARE-LoRA reduces the activation memory required for LoRA fine-tuning while keeping both LoRA matrices trainable, lowering the hardware cost of adapting large pretrained models to new tasks. This can make fine-tuning large language and diffusion models more accessible to research groups and practitioners with limited computational resources. As with other efficiency improvements for model fine-tuning, easier access to adaptation could also lower the barrier for misuse, and we encourage use of CARE-LoRA consistent with the license and usage terms of the underlying pretrained models.

\textbf{Limitations.} While CARE-LoRA is effective across the settings studied in this paper, several directions remain open. (i) Although the experiments in this paper cover T5-Base, Mistral-7B-v0.3, and SD3-Medium, evaluating larger backbones and additional modalities remains to be explored. (ii) The subspace-movement guarantee proved in this paper holds generically. This is a common form of guarantee in this line of analysis, though it does not cover atypical worst-case configurations. (iii) The reconstruction is developed for linear layers adapted by LoRA, and extending a similar idea to other adapter families or to activations retained by nonlinear operators is a direction we leave for future work.

\end{document}